\documentclass[letterpaper,11pt]{article}

\usepackage[utf8]{inputenc}
\usepackage{amssymb}
\usepackage[margin=1in]{geometry}
\usepackage{amsmath}
\usepackage{amssymb}
\usepackage{amsthm}
\usepackage{bbm}
\usepackage{amsfonts}
\usepackage{dsfont}
\usepackage{graphicx}
\usepackage{multirow}
\usepackage{color}
\usepackage{xcolor}
\usepackage{natbib}
\usepackage{hyperref}
\usepackage{graphicx}
\usepackage{subfigure}
\hypersetup{%
   breaklinks,%
   colorlinks=true,%
   linkcolor=black,%
   urlcolor=black,%
   citecolor=[rgb]{0,0,0.45}
}
\usepackage{enumitem}
\usepackage{algorithm}
\usepackage{algorithmic}

\makeatletter
\def\namedlabel#1#2{\begingroup
    #2%
    \def\@currentlabel{#2}%
    \phantomsection\label{#1}\endgroup
}
\makeatother

\theoremstyle{plain}
\newtheorem{theorem}{Theorem}[section]

\newtheorem{lemma}[theorem]{Lemma}
\newtheorem{claim}[theorem]{Claim} 
\newtheorem{corollary}[theorem]{Corollary}
\theoremstyle{definition}

\newtheorem{observation}[theorem]{Observation}
\newtheorem{remark}[theorem]{Remark}

\def\opt{\textsc{OPT}}
\def\scllp{\textsc{SimplifiedLP}}
\def\cllp{\textsc{FairRangeLP}}
\def\script#1{\mathcal{#1}}
\def\sV{\script{V}}
\def\sR{\script{R}}
\def\sP{\script{P}}
\def\sB{\script{B}}
\def\nn{\textsf{NN}}
\def\excess{\boldsymbol{f}}
\def\slack{\boldsymbol{slack}}

\title{Approximation Algorithms for Fair Range Clustering}
\author{S\`edjro S. Hotegni\thanks{African Institute for Mathematical Sciences--Rwanda. Email: \href{salomon.hotegni@aims.ac.rw}{salomon.hotegni@aims.ac.rw}} \and Sepideh Mahabadi\thanks{Microsoft Research--Redmond. Email: \href{smahabadi@microsoft.com}{smahabadi@microsoft.com}} \and Ali Vakilian\thanks{Toyota Technological Institute at Chicago (TTIC). Email: \href{mailto:vakilian@ttic.edu}{vakilian@ttic.edu}}}
\date{}
\begin{document}
\newcommand{\eps}{\varepsilon}%
\def\bar{\overline}
\def\floor#1{\lfloor {#1} \rfloor}
\def\ceil#1{\lceil {#1} \rceil}
\def\script#1{\mathcal{#1}}
\def\etal{\text{et al.}\xspace}
\def\Union{\bigcup}
\def\sep{\;|\;}
\def\card#1{|#1|}
\def\set#1{\{#1\}}
\def\E{\mathbf{E}}
\def\Var{\mathbf{Var}}

\newcommand{\argmax}{\mathrm{argmax}{}}
\newcommand{\argmin}{\mathrm{argmin}{}}

\def\wrt{\text{w.r.t.}\xspace}

\newcommand{\polylog}{\mathop{\mathrm{polylog}}}%
\newcommand{\poly}{\mathop{\mathrm{poly}}}%

\def\whp{\text{w.h.p.}}
\newcommand{\tldO}{\widetilde{O}}%
\newcommand{\tldOmega}{\widetilde{\Omega}}%
\newcommand{\tldTheta}{\widetilde{\Theta}}%

\def\sA{\script{A}}
\def\sB{\script{B}}
\def\sC{\script{C}}
\def\sD{\script{D}}
\def\sE{\script{E}}
\def\sF{\script{F}}
\def\sG{\script{G}}
\def\sI{\script{I}}
\def\sL{\script{L}}
\def\sK{\script{K}}
\def\sM{\script{M}}
\def\sN{\script{N}}
\def\sO{\script{O}}
\def\sP{\script{P}}
\def\sQ{\script{Q}}
\def\sR{\script{R}}
\def\sS{\script{S}}
\def\sT{\script{T}}
\def\sU{\script{U}}
\def\sV{\script{V}}
\def\sX{\script{X}}
\def\sY{\script{Y}}
\def\sW{\script{W}}
\def\sZ{\script{Z}}

\def\NP{\ensuremath{\mathrm{\mathbf{NP}}}}
\def\dist{\mathrm{dist}}
\def\lp{\script{LP}}
\def\scllp{\textsc{SimplifiedLP}}
\def\stllp{\textsc{StructuredLP}}
\def\stlp{\tiny{\mathrm{STLP}}}
\def\cllp{\textsc{FairRangeLP}}
\def\flp{\tiny{\mathrm{FLP}}}
\def\clla{\textsc{FacilityMatAlg}}
\def\cllc{\textsc{CenterPartAlg}}
\def\optcllp{BasicClusterLP}
\def\opt{\textsc{OPT}}
\def\rem{\mathrm{rem}}
\def\kcenter{k\mathrm{center}}

\newcommand\faircost{{\operatorname{fair-cost}}}
\newcommand\cost{{\operatorname{cost}}}
\newcommand\vol{{\operatorname{vol}}}
\def\R{\script{R}}
\newcommand{\sol}{\textsc{SOL}}%
\def\Y{\boldsymbol{Y}}
\def\X{\boldsymbol{X}}
\def\g{{\boldsymbol{g}}}
\def\solsp{\mathcal{P}}
\def\intSolSp{\mathcal{Q}}

\def\nn{\textsf{NN}}
\def\inc{\boldsymbol{inc}}
\def\slack{\boldsymbol{slack}}
\def\excess{\boldsymbol{f}}
\def\balpha{\boldsymbol{\alpha}}
\def\bbeta{\boldsymbol{\beta}}
\def\FacMat{\textsc{FacilityMatroid}}
\def\FairClust{\textsc{FairKClustering}}

\newcommand\ali[1]{\textcolor{red}{AV: #1}}

\maketitle
\thispagestyle{empty}
\begin{abstract}
    This paper studies the {\em fair range clustering} problem in which the data points are from different demographic groups and the goal is to pick $k$ centers with the {\em minimum clustering cost} such that each group is at least {\em minimally represented} in the centers set and no group {\em dominates} the centers set. 
    More precisely, given a set of $n$ points in a metric space $(P,d)$ where each point belongs to one of the $\ell$ different demographics (i.e., $P = P_1 \uplus P_2 \uplus \cdots \uplus P_\ell$) and a set of $\ell$ intervals $[\alpha_1, \beta_1], \cdots, [\alpha_\ell, \beta_\ell]$ on desired number of centers from each group, the goal is to pick a set of $k$ centers $C$ with minimum $\ell_p$-clustering cost (i.e., $(\sum_{v\in P} d(v,C)^p)^{1/p}$) such that for each group $i\in \ell$, $|C\cap P_i| \in [\alpha_i, \beta_i]$. In particular, the fair range $\ell_p$-clustering captures fair range $k$-center, $k$-median and $k$-means as its special cases.
    In this work, we provide efficient constant factor approximation algorithms for fair range $\ell_p$-clustering for all values of $p\in [1,\infty)$. 
\end{abstract}

\section{Introduction}
In recent years, the centroid-based clustering problem has been studied extensively from the fairness point of view. As in the human-centric applications, the input data usually comes from different demographics and the solution has supposedly some (in many cases, long-lasting) effects on the participants (e.g., college admissions, loan applications, and criminal justice), it is crucial to take into account the societal implications of the solution output by large scale automated processes in use. Specifically, since clustering is commonly used as a prepossessing step for more complicated ML pipelines, it is both easier and more effective to handle fairness consideration and bias reduction earlier in the pipeline rather than later.

Fair clustering was first studied by~\citet{chierichetti2017fair} and since then it has been studied with respect to various notions of fairness~\citep{bera2019fair,jung2019center,mahabadi2020individual,ahmadi2022individual,chen2019proportionally,abbasi2020fair,ghadiri2020fair,brubach2020pairwise,brubach2021fairness}. Motivated by the application of centroid-based clustering as a means of data summarization (e.g.,~\citep{moens1999abstracting,girdhar2012efficient}) for socioeconomic data,~\citet*{kleindessner2019fair} studied a notion of fair clustering in which the points belong to disjoint protected groups $P_1, \cdots, P_\ell$, and the goal is to pick exactly $k_i$ centers from each population $P_i$, s.t. it minimizes the $k$-center clustering cost (i.e., maximum distance of any point to the selected centers), where $k=\sum_{i\in [\ell]}k_i$.   

To give an example, consider an image search system. In practice, when a query is made, the user will only check the first few images output by the system. Those images (say, the first $k$ ones) act as a summary or representative of the relevant images to the searched query. Notably,~\cite{kay2015unequal} observed that in a few jobs, including CEO, women are significantly underrepresented in Google image search results.\footnote{More recent studies also show that this phenomenon of unfairness in image search results is not fully resolved yet, e.g.,~\citep{feng2022has}.} An approach to get around this disparity, as proposed by~\citet{kleindessner2019fair}, is to force the solution to contain exactly $k_i$ examples (or representative) from each protected group $i$, where $k_i$s are input parameters. Although this {\em strict} center selection requirement seems to be a plausible fix for the unfairness issue, it may incur a significant loss in the quality of the output solution! See Figure~\ref{fig:motivation}.      

\begin{figure}[!h]
    \centering
    \includegraphics[scale=0.9]{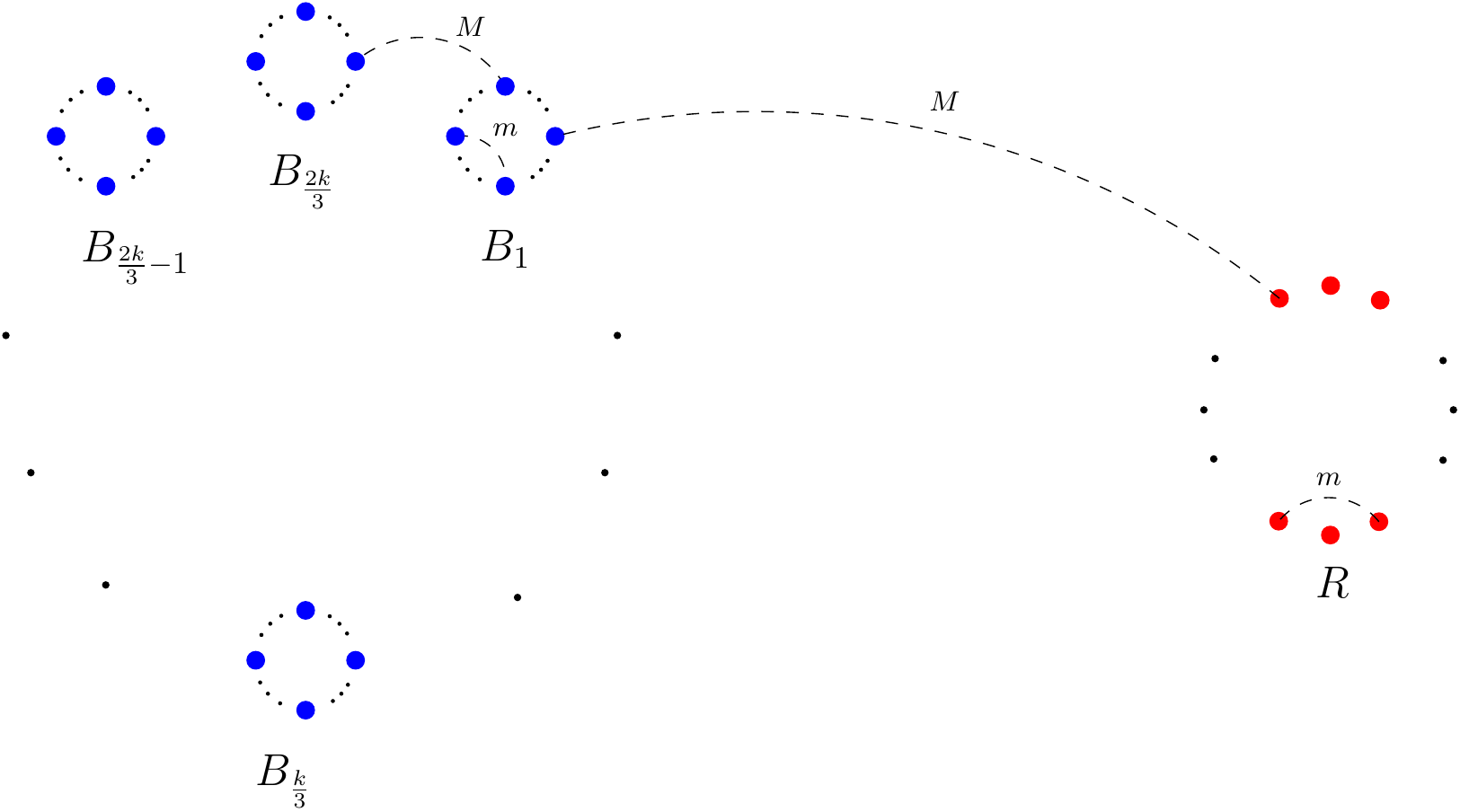}
    \caption{In this example, $\frac{n}{2}$ points with pairwise distance $m$ belong to the red group on the right side of the figure. The other $\frac{n}{2}$ points belong to the blue group on the left side, partitioned into groups $B_1, \cdots, B_{\frac{2k}{3}}$ each of size $\frac{3n}{4k}$. For every $i\neq j\in [\frac{2k}{3}]$, the pairwise distance of $v,u \in B_i$ is $m$ and the pairwise distance of $v\in B_i, u\in B_j$ is $M \gg m$. Moreover, the pairwise distance of any red and blue point is $M$. Note that the described distance function is metric. In the ``strict'' fair $k$-center, where $\frac{k}{2}$ centers should be picked from each of the red and blue group, the optimal solution has cost $M$. However, if we relax the requirement as in the fair range clustering and require $[\frac{k}{3}, \frac{2k}{3}]$ centers from each group, the $k$-center cost significantly reduces from $M$ to $m$ (the solution corresponds to picking one center from every $B_i$, for $i\in [\frac{2k}{3}]$ and arbitrary $\frac{k}{3}$ centers from $R$). While the latter solution is admissibly fair too, its quality is significantly better.}     
    \label{fig:motivation}
\end{figure}

In this paper, we study a relaxed requirement, called {\em fair range}, which only requires the number of selected centers from each protected group to be in a given interval specified by a lower bound and an upper bound. As we can easily tolerate slight deviations from the ``expected" presence of each protected group, fair range is a more natural requirement and has a better alignment with practice, compared to the strict requirement. 
On the other hand, the larger the requirement interval becomes, the better the quality of the clustering becomes.
While fair clustering with strict (i.e., exactly $k_i$ from each group $i$) requirement---{\em which itself is a special case of clustering under partition matroid constraints}---is a well-studied problem in the domain of fair clustering ~\citep{hajiaghayi2010budgeted,krishnaswamy2011matroid,charikar2012dependent,swamy2016improved,chen2016matroid,krishnaswamy2018constant,jones2020fair,chiplunkar2020solve}, this natural generalization of the problem has not been studied in the context of fair clustering. We remark that, very recently,~\citet{nguyen2022fair} studied fair range $k$-center. They studied the problem both in the standard offline and the streaming models and provided constant factor approximations in both regimes. Moreover,~\citet{thejaswi2021diversity} studied the problem when we are provided with lower bounds only and the objective is $k$-median cost. They showed fixed parameter algorithms for this problem which is referred to as {\em diversity-aware $k$-median}. 
However, the approximability of the general problem of fair range clustering with the $\ell_p$-objective which includes the standard $k$-median and $k$-means still remains open.



\paragraph{Our contributions.}
In this paper, we study the fair range clustering with the $\ell_p$-objective and provide a constant factor approximation algorithm for all values of $p\in [1,\infty)$. More precisely, our main result is as follows. 
\begin{theorem}\label{thm:main}
For all $p\in [1,\infty)$, there exists a constant factor approximation algorithm for fair range $k$-clustering with the $\ell_p$-objective that runs in polynomial time. 
\end{theorem}
More generally, our algorithm works for fair range facility location with zero opening cost facilities which generalizes fair range clustering.

As fair range $k$-clustering with $\ell_p$-objective has $k$-center as its special case ($k$-center is equivalent to fair range $k$-center where for every $i\in [\ell]$, $\alpha_i = 0$ and $\beta_i = k$), it is NP-hard to approximate it within a factor better than $2$~\citep{gonzalez1985clustering}. So, our approximation factor for the problem is tight up to a constant factor.

To design our algorithm, we start with the framework of~\citep{charikar2002constant,krishnaswamy2011matroid} to sparsify the input instance and find a feasible fractional solution of the standard LP-relaxation of the problem, called $\cllp$, on the sparse instance. However, the main difficulty is within the rounding part. 
In our algorithm, we further exploit the combinatorial structures of an approximately optimal fractional solution of $\cllp$. We consider another relaxation of the problem, $\stllp$, whose optimal solution is within an $e^{O(p)}$ factor of the optimal solution of $\cllp$.\footnote{In both relaxations, the objective is $\ell_p$-objective raised to power of $p$. So, with respect to the original objective these relaxation are within a constant factor of each other.} Then, using the techniques from Polyhedral Combinatorics, we show that $\stllp$ has a half-integral optimal solution. Next, by a reduction to the network flow problem, we design a combinatorial algorithm that outputs an $e^{O(p)}$-approximation integral solution of $\stllp$, relying on the properties of a half-integral optimal solution of the LP. Finally, we can convert this feasible integral solution of $\stllp$ to a feasible integral solution of $\cllp$ without losing more than a constant factor in the approximation ratio.

Refer to Section~\ref{sec:decription} for an overview of our algorithm and a summary of the first part of our algorithm which is standard and also used in some of prior works on (fair) clustering. Then, in Section~\ref{sec:rounding}, we describe our efficient rounding approach for fair range clustering.   

\paragraph{Other related work.}
Clustering has been an active area of research in the domain of fairness for algorithms and machine learning. In particular, it has been studied under various settings including group fairness notions such as {\em fair representation}~\citep{chierichetti2017fair,bera2019fair,bercea2019cost,backurs2019scalable,ahmadian2019clustering,dai2022fair}, {\em social fairness}~\citep{abbasi2020fair,ghadiri2020fair,makarychev2021approximation,chlamtavc2022approximating,ghadiri2022constant}, {\em proportional fairness}~\citep{chen2019proportionally,micha2020proportionally} and {\em individual fairness}~\citep{jung2019center,mahabadi2020individual,negahbani2021better,vakilian2022improved,ahmadi2022individual,brubach2020pairwise,brubach2021fairness}.

A similar notion has been studied for the related problem of nearest neighbor search~\citep{har2019near,aumuller2020fair,aumuller2021sampling} and submodular maximization~\citep{el2020fairness}.

\begin{remark}\label{rem:matroid-clustering}
    While the collection of solutions satisfying the given range constraints does not constitute a matroid (since these constraints do not satisfy the downward-closed property),~\citep{el2020fairness} showed that the family of all ``extendable'' subsets is indeed a matroid.\footnote{The work of~\citep{el2020fairness} used this idea to study the fairness of an objective different from clustering, namely submodular maximization.} In our context, a set of facilities $\Tilde{F}$ is extendable if it is a subset of facilities $\Bar{F}$ where $\Bar{F}$ has size at most $k$ and satisfies all given range constraints. More precisely, $\Tilde{F}$ is extendable if and only if for all $i\in [\ell]$, $|\Tilde{F}\cap P_i|\le \beta_i$ and $\sum_{i\in [\ell]} \max\{|\Tilde{F}\cap P_i|, \alpha_i\} \le k$. Then, one can use an existing algorithm for the $k$-clustering under matroid constraint (e.g.,~\citep{krishnaswamy2011matroid,swamy2016improved,krishnaswamy2018constant}) and find a constant factor approximation for fair range clustering. While the above algorithm only considers the $k$-median objective, it is known that this approach can be generalized to any value of $p \in [1,\infty)$, see~\citep{vakilian2022improved}.

    While working with the extendable sets reduces our problem to matroid $k$-clustering, the known algorithm will require solving LPs with exponentially many constraints. Although these LPs can be solved in polynomial time using the ellipsoid algorithm, the best known algorithm for solving such LPs, which is via cutting plane methods, runs in time $O(n^{3}\log(1/\epsilon))$ and returns a $(1+\epsilon)$-approximate solution~\citep{jiang2020improved}. However, in our approach, we work with small size LPs of the fair range clustering that can be solved in time $O(n^{1.5}k^{1.5})$ via the interior point method~\citep{van2021minimum}. So, arguably our algorithm is more efficient (in particular, for the case $k= O(1)$, we obtain a quadratic improvement in the run-time).
\end{remark}
\section{Description of Our Algorithm}\label{sec:decription}
For the sake of clarity, similarly to~\citep{swamy2016improved} on clustering under matroid constraints, we describe our algorithm for a more general problem of fair range facility location with the $\ell_p$-objective in which the number of open facilities are required to be (at most) $k$. 
For our application of clustering, it suffices to consider the instances of fair range facility locations in which the set of {\em facilities} $F$ is disjoint from the set of {\em clients}, where the clients are specified by a pair of {\em set of locations} $D$ and a {\em demand function} $w: D\rightarrow \mathbb{Z}_{> 0}$.Note that while $F$ and $D$ are disjoint, they may have points at the same location. 

More specifically, let $(P,d)$ be a metric space. In this problem, we are given a set of facilities $F:=F_1\uplus\cdots \uplus F_\ell$ where $F\subseteq P$, a set of integral range parameters $\{\alpha_i, \beta_i\}_{i\in [\ell]}$ where $\alpha_i,\beta_i\in \mathbb{Z}_{\geq 0}$, a set of locations $D\subseteq P$, and a demand function $w\colon D\rightarrow \mathbb{Z}_{> 0}$ denoting the {\em total} number of clients in each location of $D$.\footnote{The integrality of $\alpha_i, \beta_i$ are without loss of generality as we can always replace them with $\lceil{\alpha_i\rceil}$ and $\lfloor \beta_i \rfloor$ and the solution space remains the same.} The goal is to open a subset $C\subseteq F$ of $k$ facilities with minimum $\ell_p$-cost assignment of clients to their closest facilities (i.e., $(\sum_{v\in D} w(v)\cdot d(v,C)^p)^{1/p}$), such that for every group $i\in [\ell], |C\cap F_i| \in [\alpha_i, \beta_i]$.

\begin{remark}
In the description of our algorithm and in particular, in the objective of the LP-relaxations, we consider the $\ell_p$-clustering raised to the power of $p$. Throughout the paper, we provide all approximation factors according to $\sum_{v\in D} w(v)\cdot d(v,C)^p$ denoted as $\cost$. Then, at the end, to derive the result of the main theorem, we raise the approximation factor to $1/p$.
\end{remark}

Our algorithm has two major components: (1) finding an approximate fractional solution $(x,y)$ for an LP-relaxation of the problem with certain {\em well-separatedness} and {\em locality} properties, (2) rounding the fractional solution to an integral solution without losing more than a factor of $e^{O(p)}$ in the LP objective. The first part is essentially a standard approach in approximation algorithm for clustering, introduced in~\citep{charikar2002constant}, and is similar to the one used in~\citep{krishnaswamy2011matroid,swamy2016improved}. So, here we only describe the properties of the fractional solution at the end of the first part and for a detailed exposition, we refer to Appendix~\ref{sec:preprocessing}.  

The main technical contribution of our paper is an efficient rounding algorithm that preserves the clustering with $\ell_p$-objective (for all $p\in [1,\infty)$) up to a constant factor, does not violate any of fairness constraints, and opens at most $k$ centers.

\subsection{Reducing the number of locations}
Given a set of points $P$, we first run an existing efficient constant factor approximation algorithm for $k$-clustering with $\ell_p$ objective (ignoring all range constraints) to get a set of centers $C = (c_1, \cdots, c_k)$. Then, we construct the following instance of fair range clustering. We separate the set of clients ($D$) and facilities ($F$) as follows: Each point in $P$ is moved to its closest center in $C$ and the resulting set of points located at $k$ locations $c_1, \cdots, c_k$ constitutes the set of clients. In other words, the set of clients can be described in terms of $k$ locations plus the total number of clients in each location determined by $w' : D \rightarrow \mathbb{Z}_{> 0}$. For the facilities, we set $F = P$. Then, we solve the minimum cost fair-range clustering (or facility location w.r.t. $D$ and $F$) that picks $k$ centers (or facilities) from $F$ and serves all clients. 

\begin{theorem}\label{thm:location-reduction}
Given an $O(\alpha)$-approximation algorithm of $k$-clustering with $\ell_p$-objective, a $\beta$-approximate solution $S$ of fair range clustering with $\ell_p$-objective on the described instance ($D, F$) is an $O(\alpha\beta)$-approximation for fair range clustering with $\ell_p$-objective on the original instance $P$.   
\end{theorem}
\begin{proof}
As the set of facilities $F = P$, the set $S$ is a feasible solution for fair range clustering on the original instance $P$ too. 

Let $D = \{c_1, \cdots, c_k\}$ denote an $\alpha$-approximate solution of clustering with $\ell_p$-objective on $P$.
Let $\opt_{\mathrm{org}}$ and $\opt_{\mathrm{rdc}}$ respectively denote optimal solutions of fair range clustering with $\ell_p$-objective on the original instance $P$ and the reduced instance $(D, F)$. First we bound the cost of $\opt_{\mathrm{org}}$ on the reduced instance $(D,F)$:
\begin{align}
    \cost_{\mathrm{rdc}}(\opt_{\mathrm{org}}) 
    &= \sum_{v\in D} w'(v)\cdot d(v, \opt_{\mathrm{org}})^p \nonumber\\
    &\le \sum_{v\in P} w(v) \cdot 2^{p-1} \cdot \big(d(v, D)^p + d(v, \opt_{\mathrm{org}})^p\big) \; \rhd\text{approximate triangle inequality}\nonumber \\
    &\le 2^{p-1}\big(\sum_{v\in P} w(v) d(v,D)^p  + \sum_{v\in P} w(v) d(v,\opt_{\mathrm{org}})^p\big) \nonumber \\
    &\le 2^{p-1} \cdot (\alpha^p \cdot \cost(\opt_{\mathrm{org}}) + \cost(\opt_{\mathrm{org}})) \; \rhd\text{$D$ is an $\alpha$-approximate solution}\nonumber \\
    &= 2^{p-1}(\alpha^p +1) \cdot \cost(\opt_{\mathrm{org}}) \label{eq:org-opt-cost-on-rdc}
\end{align}
Next, we bound the cost of the solution $S$ on the original instance $P$,
\begin{align*}
    \cost(S)
    &= \sum_{v\in P} w(v) \cdot d(v, S)^p \\
    &= \sum_{v\in P} w(v) \cdot 2^{p-1} \cdot \big(d(v, D)^p + d(\nn_{D}(v), S)^p \big) &&\rhd\text{approximate triangle inequality}\\
    &\le 2^{p-1}\big(\sum_{v\in P} w(v)d(v,D)^p + \sum_{v\in D} w'(v)d(v,S)^p\big) \\
    &\le 2^{p-1}\cdot (\alpha^p \cdot \cost(\opt_{\mathrm{org}}) + \beta^p \cdot \cost_{\mathrm{rdc}}(\opt_{\mathrm{rdc}})) \\\
    &\le 2^{p-1} \cdot (\alpha^p + \beta^p \cdot 2^{p-1} (\alpha^p +1)) \cdot \cost(\opt_{\mathrm{org}}) &&\rhd\text{by Eq~\eqref{eq:org-opt-cost-on-rdc}}
\end{align*}
Thus, $S$ is an $O(\alpha\beta)$-approximation for fair range clustering with $\ell_p$-objective on the original instance $P$.
\end{proof} 

In the rest of this paper, at the expense of losing a factor of  $O(\alpha)$ in the approximation factor and running an $\alpha$-approximation algorithm of $k$-clustering with $\ell_p$-objective (with no range constraints), we assume that $|D| = k$ and $|F| = n$.  

\subsection{Constructing a structured fractional solution}
In this section, we describe the sparsification approach of~\cite{charikar2002constant} which outputs an instance with a subset of locations where pairwise distances of survived locations are ``relatively large''. 
First, we state a standard LP-relaxation of the problem as follows.
\begin{align}
&\cllp(D, F, w, \{\alpha_i, \beta_i\}_{i\in [\ell]})\nonumber\\[1mm]
\text{minimize }& \ \rlap{$\sum_{v\in D, u\in F} w(v) \cdot d(v, u)^p \cdot x_{vu}$} \nonumber\\[1mm]
\text{s.t.}\qquad &\sum_{u\in F} x_{vu} \geq 1 \quad\forall v\in D \label{cst:assgn} \\ 
&\alpha_i \leq \sum_{u\in F_i} y_{u} \leq \beta_i \quad\forall i\in [\ell]  \label{cst:fairness}\\
&\sum_{u\in F} y_u \leq k \label{cst:k-center}\\
&0 \leq x_{vu} \leq y_{u}  \quad\forall v\in D, u\in F \label{cst:under_assignment}
\end{align}
In our algorithm, we never modify the set $F$ and the fairness constraints $\{\alpha_i, \beta_i\}_{i\in [\ell]}$. So, to specify an instance, we will provide the set of clients $(D, w)$.
Consider an optimal fractional solution $(x^*,y^*)$ of $\cllp(D, w)$. 
For any location $v\in D$, define $\sR(v):= \big(\sum_{u\in F} x^*_{vu} \cdot d(v,u)^p\big)^{1/p}$ as the {\em fractional distance} of a unit of demand at location $v$ w.r.t. the optimal solution $(x^*,y^*)$. 
When $y^*$ is integral, then $\sR(v)$ is the distance of $v$ to its closest open facility, specified by the vector $y^*$.

The sparsification approach of~\cite{charikar2002constant} applied to our setting, described in Appendix~\ref{sec:preprocessing}, outputs a sparse instance with the following properties. The proofs of theorems in this sections are deferred to Appendix~\ref{sec:proof-sec-2}. 
\begin{theorem}\label{thm:sparsification-step}
    Given an instance $(D,w)$ of fair range clustering with $\ell_p$-cost and an optimal fractional solution $(x,y)$ of $\cllp(D, w)$ with cost $\opt_D$, there exists a polynomial time algorithm that returns a set of locations $D' \subseteq D$ and a demand function $w': D' \rightarrow \mathbb{R}$ such that
    \begin{enumerate}
        \item[\namedlabel{prop:separatedness}{(Q1)}] For every pair of $v_i, v_j$ in $D'$, $d(v_i, v_j) \ge 2^{1+1/p} \max\{\sR(v_i), \sR(v_j)\}$.
        \item[\namedlabel{prop:feasibility}{(Q2)}] $(x,y)$ is a feasible solution of $\cllp(D', w')$ of cost at most $\opt_D$,
        \item[\namedlabel{prop:cost-preservation}{(Q3)}] Any integral solution $C$ of $\cllp(D', w')$ of cost $z$, can be converted in polynomial time to a feasible solution of $\cllp(D, w)$ of cost at most $4^p \cdot \opt_D + 2^{p-1} \cdot z$. 
    \end{enumerate}
\end{theorem}

Next, for every location $v\in D'$, we define the {\em ball} $\sB(v):= \{u\in F | d(v, u) \le 2^{\frac{1}{p}}\cdot\sR(v)\}$ to denote the set of facilities at distance at most $2^{\frac{1}{p}}\cdot\sR(v)$ from $v$. Further, for every location $v\in D'$, we define $\sP(v)$ as the super ball of $v$ which consists of $\sB(v)$ and a set of ``private facilities'' of $v$. 

\begin{observation}\label{obs:sparse-dist-rel}
For any pair of $v, v'\in D'$ and $u'\in \sR(v')$, $\frac{1}{2} d(v,v') \le d(v, u') \le \frac{3}{2} d(v,v')$. 
\end{observation}
\begin{proof}
Since $d$ is a metric distance, by the triangle inequality, $d(v,u') + d(u', v') \ge d(v, v')$.
Next, we bound $d(u', v')$ in terms of $d(v,v')$. Since $u'\in \sB(v')$, $d(v', u') \le 2^{\frac{1}{p}} \sR(v')$. On the other hand, by~\ref{prop:separatedness}, $d(v,v') \ge 2^{1+\frac{1}{p}}\sR(v')$. Hence, $d(v', u') \le \frac{1}{2} d(v,v')$. 
By another application of the triangle inequality, $d(v, u') \le d(v, v') + d(v', u') \le \frac{3}{2} d(v,v')$.
\end{proof}

Next, in polynomial time, we convert a fractional solution of $\cllp(D,w)$ to a fractional solution $(x, y)$ of $\cllp(D',w')$ with the following further structural properties.  

\begin{theorem}\label{thm:structured-fractional}
    There exists a polynomial time algorithm that outputs a fractional solution $(x,y)$ of $\cllp(D', w')$ of cost $9^{p} \cdot \opt_D$, where $\opt_D$ is the cost of an optimal solution of $\cllp(D, w)$, and a collection of super balls $\{\sP(v)\}_{v\in D'}$ that satisfy the following properties:
    \begin{enumerate}
        \item[\namedlabel{prop:ball-inclusion}{(P1)}] For every $v\in D'$, $\sB(v) \subseteq \sP(v)$,
        \item[\namedlabel{prop:optimality}{(P2)}] For every $v\in D'$ and $u\in \sP(v)\setminus \sB(v)$, $x_{vu}>0$ only if $\sum_{u\in \sB(v)}y_{u} <1$. Similarly, for every $v\in D'$ and $u\in F\setminus \sP(v)$, $x_{vu}>0$ only if $\sum_{u\in \sP(v)} y_{v} < 1$.    
        \item[\namedlabel{prop:external-center}{(P3)}] For every $v\in D'$, if $x_{vu} > 0$, then either $u\in \sP(v)$ or $u\in \sB(v')$ where $v' = \nn_{D'}(v)$ denotes the nearest location in $D'$ (other than $v$ itself) to $v$, 
        \item[\namedlabel{prop:half-center}{(P4)}] For every $v\in D'$, \[\sum_{u\in \sP(v)} x_{vu} \ge \sum_{u\in \sB(v)} x_{vu} \ge 1/2.\]
        \item[\namedlabel{prop:near-center-further-than-private-center}{(P5)}] For every $v\in D'$, $u\in \sP(v)\setminus \sB(v)$, $d(v,u) \le 2\cdot d(v,v')$, where $v' = \nn_{D'}(v)$.
        \item[\namedlabel{prop:disjoint-superballs}{(P2)}] The set of super balls, $\{\sP(v)\}_{v\in D'}$, are disjoint.
    \end{enumerate}
\end{theorem}

\section{Rounding Algorithm}\label{sec:rounding}
To do the rounding, first we write a new LP relaxation, called $\stllp$ 
which is a simplification of $\cllp$ via the further structures of fractional solution $(x,y)$ guaranteed by Theorem~\ref{thm:structured-fractional}. In particular, $\stllp$ is useful for our rounding algorithm because as we show in this section, the polyhedraon constructed by the constraints of $\stllp$ is half-integral. A solution $y$ to an LP relaxation is half-integral if every coordinates in $y$ has value either $0, 1/2$ or $1$.

In $\stllp$, we define $\Delta(v) := d(v, v')^p + \sum_{u\in \sP(v)} \big(d(v, u)^p - d(v, v')^p\big) \cdot y_u$ to denote the minimum (fractional) distance of $v$ to open facilities (i.e., the facilities vector $y$).

\begin{align}
    &\stllp(D', F, w', \{\alpha_i, \beta_i\}_{i\in [\ell]})\nonumber\\[1mm]
    \text{minimize }& \rlap{$\sum_{v\in D'} w'(v) \cdot \Delta(v)$} \nonumber\\[1mm]
    \text{s.t.}\qquad &\alpha_i \leq \sum_{u\in F_i} y_{u} \leq \beta_i &&\forall i\in [\ell]  \label{cst:str-fairness}\\
    &\sum_{u\in F} y_u \leq k \label{cst:str-k-center}\\
    &\sum_{u\in \sB(v)} y_{u} \ge 1/2  &&\forall v\in D' \label{cst:str-half-assignment} \\
    &\sum_{u\in \sP(v)} y_u \le 1 &&\forall v\in D' \label{cst:str-full-assignment} \\
    &y_u \ge 0 &&\forall u\in F
\end{align}
\begin{lemma}\label{lem:str-lp-opt-bound}
    The optimal fractional solution of $\stllp(D', w')$ is a valid solution for fair range clustering on $(D', w')$ and has cost at most $e^{O(p)} \cdot \opt_D$.
\end{lemma}
\begin{proof}
        By the cardinality constraint on the number of open facilities and the range constraints on the number of facilities from each group, a feasible solution of $\stllp$ is a feasible solution of fair range clustering on $(D', w')$. 
    
    Consider the solution $(x,y)$ guaranteed by Theorem~\ref{thm:structured-fractional}. It is straightforward to verify that $y$ satisfies all constraints of $\stllp(D', w')$ and is a feasible solution for the LP: $(y)$ satisfies the first two sets of constraints in $\stllp(D', w')$ because $(x,y)$ is a feasible solution of $\cllp(D', w')$ and constraint~\ref{cst:str-half-assignment} follows from~\ref{prop:half-center} in Theorem~\ref{thm:structured-fractional}. Note that if $y$ does not satisfy constraint~\eqref{cst:str-full-assignment}, we can decrease the value of $y$ accordingly so that $(x,y)$ remains a feasible solution of $\cllp(D', w')$ with a lower cost.
    
    Next, we bound the cost of $y$ with respect to $\stllp$, $\cost_{\stlp}(y)$. 
    We rewrite $\Delta(v)$ as follows:
    \begin{align*}
        \Delta(v) = \big(1 - \sum_{u\in \sP(v)} y_u \big) d(v, v')^p + \sum_{u\in \sP(v)} d(v, u)^p \cdot y_u
    \end{align*}
    For every $v\in D'$,
    \begin{align}
        \sum_{u\in F} d(v, u)^p x_{vu} 
        &=\sum_{u\in \sP(v)} d(v,u)^p x_{vu} + \sum_{u'\in \sB(v')} d(v, u')^p x_{vu} &&\rhd\text{by~\ref{prop:external-center}}\nonumber \\
        &\ge \sum_{u\in \sP(v)} d(v,u)^p x_{uv} + \frac{1}{2^p} \cdot d(v,v')^p \sum_{u\in \sB(v')} x_{vu} &&\rhd\text{by Observation~\ref{obs:sparse-dist-rel}} \nonumber\\
        &= \sum_{u\in \sP(v)} d(v,u)^p y_u + \frac{1}{2^p} \cdot d(v,v')^p (1 - \sum_{u\in \sP(v)} y_u) &&\rhd\text{by~\ref{prop:optimality}} \nonumber \\
        &\ge \frac{1}{2^p} \Big( \big(1 - \sum_{u\in \sP(v)} y_u \big) d(v, v')^p + \sum_{u\in \sP(v)} d(v, u)^p \cdot y_u\Big), \label{eq:up-bnd}
    \end{align}
    Thus, 
    \begin{align*}
        \cost_{\stlp}(y)
        = \sum_{v\in D'} w'(v) \Delta(v)
        \le 2^p \sum_{v\in D'} w'(v) \sum_{u\in F} d(v,u)^p x_{vu}
        = 2^p \cdot \cost_{\flp}(x,y) = e^{O(p)} \cdot \opt_D  
    \end{align*}
    where $\opt_D$ denotes the cost of an optimal solution of $\cllp(D, w)$. The last inequality bounding $\cost_{\flp}(x,y)$ is by Theorem~\ref{thm:structured-fractional}.
\end{proof}

\begin{lemma}\label{lem:str-to-fair-lp}
    Consider a half-integral solution $\tilde{y}$ of $\stllp(D', w')$ of cost $z$. Then, $\tilde{y}$ is a feasible solution for $\cllp(D', w')$ with cost at most $\big(\frac{3}{2}\big)^p \cdot z$.
\end{lemma}
\begin{proof}
        Let $\tilde{x}$ be the assignment of clients to the opened facilities as follows: For every $v\in D'$,
    \begin{itemize}[leftmargin=*]
        \item {\bf Step 1.} For every $u\in \sP(v)$, set $\tilde{x}_{vu} = \tilde{y}_u$. Let $\tilde{Y}(v) = \sum_{u\in \sP(v)} y_u$ denote the total assignment of $v$ at the end of this step.
        \item {\bf Step 2.} While $\tilde{Y}(v)<1$, iterate over $u\in \sB(v')$ and at each iteration set $\tilde{x}_{vu} = \min\{1- \tilde{Y}(v), \tilde{y}_u\}$ and $\tilde{Y}(v) = \tilde{Y}(v) + \tilde{x}_{vu}$. 
    \end{itemize}
    Since for every $v\in D'$, $\frac{1}{2} \le \sum_{u\in \sB(v)} \tilde{y}_u \le 1$, this procedure terminates with a feasible fractional assignment of clients to facilities, $\tilde{x}$: for every $v\in D', u\in F$, $\tilde{x}_{vu} \le \tilde{y}_u$ and $\sum_{u\in F} \tilde{x}_{vu} = 1$. Moreover, if $\tilde{y}$ is half-integral, clearly $\tilde{x}$ is half-integral too. 

    For every $v\in D'$,
    \begin{align}
        \sum_{u\in F} d(v,u)^p \cdot \tilde{x}_{vu}
        &=\sum_{u\in \sP(v)} d(v,u)^p \cdot \tilde{x}_{vu} + \sum_{u'\in \sB(v')} d(v,u')^p \cdot \tilde{x}_{vu'} \nonumber \\
        &\le \sum_{u\in \sP(v)} d(v,u)^p \cdot \tilde{x}_{vu} + \big(\frac{3}{2}\big)^p \sum_{u'\in \sB(v')} \tilde{x}_{vu'} &&\rhd\text{by Observation~\ref{obs:sparse-dist-rel}} \nonumber \\
        &= \sum_{u\in \sP(v)} d(v,u)^p \cdot \tilde{y}_{u} + \big(\frac{3}{2}\big)^p (1 - \sum_{u\in \sP(v)} \tilde{y}_{u}) \cdot d(v, v')^p \nonumber \\
        &\le \big(\frac{3}{2}\big)^p \big(\sum_{u\in \sP(v)} d(v,u)^p \cdot \tilde{y}_{u} +  (1 - \sum_{u\in \sP(v)} \tilde{y}_{u}) \cdot d(v, v')^p \big) \nonumber\\
        &= \big(\frac{3}{2}\big)^p \cdot \Delta(v) \label{eq:int-bound-stllp}
    \end{align}
    Thus,
    \begin{align*}
        \cost_{\flp}(\tilde{x}, \tilde{y}) 
        = \sum_{v\in D'} w'(v) \sum_{u\in F} d(v,u)^p \tilde{x}_{vu}
        &\le \sum_{v\in D'} w'(v) \cdot \big(\frac{3}{2}\big)^p \cdot \Delta(v) &&\rhd\text{by Eq.~\eqref{eq:int-bound-stllp}} \\ 
        &\le \big(\frac{3}{2}\big)^p \cdot \cost_{\stlp}(\tilde{y})
    \end{align*}
\end{proof}

\subsection{Constructing a half-integral solution of $\stllp$}
Next, we show that the constraints of $\stllp$ forms a totally unimodular (TU) matrix. Then, by scaling the constraint properly, we can find a half-intergal optimal solution of $\stllp$ in polynomial time. Recall that a matrix is TU if every square submatrix has determinant $-1, 0$ or $1$. 
Totally unimodular matrices are extremely important in polyhedral combinatorics and combinatorial optimization since if $A$ is TU and $b$ is integral, then the linear programs of the form $\{\min cx \;|\; Ax \ge b, x\ge 0\}$ has integral optima, for any cost vector $c$.
\begin{lemma}\label{lem:tu-matrix}
The matrix corresponding to the constraints of $\stllp$ is a TU matrix.    
\end{lemma}
\begin{proof}
We consider $\stllp$ in the form $\{\min cy \;|\; Ay \ge b, y\ge 0\}$. 
Note that values of all entries in $A$ belong to $\{0, \pm 1\}$.
Moreover, in each column $A_i$, corresponding to each variable $y_i$, there are at most five non-zero entries: one with value $-1$ corresponding to ``picking at most $k$ centers'' constraint, one with value $-1$ to upper bound the contribution of the facilities of the color class to which facility $i$ belongs, one with value $1$ to lower bound the contribution of the facilities of the color class to which facility $i$ belongs, one with value $1$ to open at least $1/2$ facilities from the super ball containing facility $i$ and finally one with value $-1$ to open at most $1$ from the super ball containing facility $i$.  

Now, we show that $A$ is a TU matrix by the result of~\citet{ghouila1962caracterisation}; see Theorem~\ref{thm:Ghouila-Houri}. 
They showed that if for every subset $R$ of the rows there is an assignment $s_R:R \rightarrow \{-1, 1\}$ of signs to rows so that $\sum_{r\in R}s_R(r) A_r$ has all entries in $\{0, \pm1\}$, then $A$ is a TU matrix. To show that, for any subset of rows $R$ we write $R = R_{\textrm{center}} \cup R_{\textrm{upper}} \cup R_{\textrm{lower}} \cup R_{\textrm{superball}}$. Next, we consider the following two cases:
\begin{itemize}[leftmargin=*]
    \item{$R_\textrm{center} \neq \emptyset$.} In this case, we assign $-1$ to the row corresponding to the total facility budget constraint,~\eqref{cst:str-k-center}. Next, we consider the color classes $C_{\textrm{both}}$ for which both rows corresponding to the upper and lower bounds constraints on the number of centers from the color class exist in $R$, $s_R$ assigns $1$ to both such rows. Note that as the rows corresponding to the lower and upper bound constraints of a specific color class sum up to zero vector, such assignment makes the contribution of such rows in the sum of rows of $A$ in $R$ zero. Otherwise, as color classes are disjoint, if exactly one of the rows corresponding to the upper and lower bound constraints belong to $R$, we can set $s_R$ so that the contribution of entries in such rows become exactly $1$. Hence, so far (by summing up terms $\sum_{r\in R_{\textrm{center}}} s_R(r)A_r$, $\sum_{r\in R_{\textrm{upper}}} s_R(r)A_r$ and $\sum_{r\in R_{\textrm{lower}}} s_R(r)A_r$), each entry has value either $0$ or $-1$. Next, we consider the set of rows in $R_{\textrm{superball}}$. If both rows corresponding to the upper and lower bounds constraints on the contributions of facilities in a super ball are in $R$, then $s_R$ assigns $1$ to both rows and their total contributions in the final sum will become zero.   
    For the remaining rows in $R_{\textrm{superball}}$, we set the sign so that the contribution of each row to the final sum becomes exactly $1$ (i.e., if the row corresponds to constraint~\eqref{cst:str-half-assignment}, $s_R$ assigns $1$, and assigns $-1$ otherwise). Since the super balls are disjoint, in the weighted sum of rows with $s_R$, each entry is either $0$ or $\pm1$.
    \item{$R_{\textrm{center}} = \emptyset$.} Our approach is similar to the previous case. However, here first we set the assignment of rows corresponding to $R_{\textrm{lower}}, R_{\textrm{upper}}$ and makes sure that by summing up $\sum_{r\in R_{\textrm{upper}}} c(r)A_r$ and $\sum_{r\in R_{\textrm{lower}}} c(r)A_r$, each entry has value either $0$ or $1$. Then, we follow a similar approach to the previous case and exploiting the fact super balls are disjoint, we show that it is possible to set the sign of rows in $R_{\textrm{superball}}$ so that in the overall signed sum of the rows, each entry is either $0$ or $\pm1$. 
\end{itemize}
\end{proof}
\begin{theorem}\label{them:str-half-integral}
    $\stllp$ has an optimal half-integral solution.
\end{theorem}
\begin{proof}
    To see this, let vector $\hat{y} = y/2$ and then note that by Lemma~\ref{lem:tu-matrix}, the matrix corresponding to the constraints of $\stllp$ is a TU matrix. Hence, the polytope specified by $A\hat{y} \le 2b$ is integral which implies that $\stllp$ admits a half integral optimal solution.
\end{proof}
We remark that such half-integral optimal solution can be found via an efficient polynomial time algorithm (see Remark~\ref{rem:matroid-clustering}).

\subsection{Constructing an Integral Solution of $\stllp$}
In this section, exploring the structure imposed by a half-integral optimal solution of $\stllp$, we construct an integral $e^{O(p)}$-approximate solution of $\stllp$.

\noindent\textbf{Step 1. Partitioning facilities.} First, we show that we can partition facilities into $L\le k$ disjoint sets $S_1, \cdots, S_L$ such that any solution that opens at least one facility from each $S_i$ for all $i\in L$ is an $e^{O(p)}$-approximate solution of $\stllp$. Here is the description of the partitioning procedure given a half-integral optimal solution of $\stllp$, $y$:

\begin{algorithm}[!ht]
	\begin{algorithmic}[1]
		\STATE {\bfseries Input:} A set of locations $D'$, half-integral vector $y$
        \FORALL{location $v_i \in D'$}
            \STATE $R_i \leftarrow$ the minimum assignment cost of a unit of demand at $v_i$ w.r.t.~$y$
            \STATE $\rhd\; R_i = \frac{1}{2} (d(v_i, u_{i_1})^p + d(v_i, u_{i_1})^p)$ where $u_{i_1}, u_{i_2}$ are respectively the primary and secondary 
            \STATE $\rhd$ facilities serving $v_i$
            \STATE $S_i \leftarrow \{u_{i_1}\} \cup \{u_{i_2}\}$
        \ENDFOR
        \STATE $D'' \leftarrow D', \bar{D} \leftarrow \emptyset$
        \WHILE{$D''$ is nonempty}
        \STATE {\bf let} $v_i \leftarrow \argmin_{v_j \in D''} R_j$
        \STATE {\bf add} $v_i$ to $\bar{D}$
        \STATE {\bf remove} all locations $v_j \in D''$ such that $S_j \cap S_i \neq \emptyset$
        \ENDWHILE
    \end{algorithmic}
	\caption{Partitioning Facilities.}
	\label{alg:partitioning_facilities}
\end{algorithm}

Next, we show the following useful property of the $\{S_i\}_{\{i \;|\; v_i\in \bar{D}\}}$.
\begin{lemma}\label{lem:approximation-bound}
    The clustering cost of any set of $k$ facilities $C$ that opens at least one facility from each $\{S_i\}_{\{i \;|\; v_i\in \bar{D}\}}$ is at most $\big(\frac{9}{2}\big)^p$ times the cost of an optimal solution of $\stllp(D', w')$.
\end{lemma}
\begin{proof}
    We compute the cost of a unit of demand in each location of $v_i \in D'$ when it is assigned to its closest facility in $\script{S} := \bigcup_{i: v_i\in \bar{D}} S_i$. In particular, we compare this cost to the fractional cost imposed by the half-integral solution $y$ of $\stllp$. We consider the following two cases:
    \begin{itemize}[leftmargin=*]
        \item{\bf $v_i \in \bar{D}$.} It is straightforward to check that in this case the assignment cost of $v_i$ to a facility in $\script{S}$ is at most twice the its fractional assignment cost with respect to $y$.
        \item{\bf $v_i \notin \bar{D}$} Let $v_j \in C$ be the client who removed $v_i$ from $D''$, i.e., $v_j$ is the minimum cost location whose opened facilities, $S_j$, has non-empty intersection with $S_i$. Let $u = S_j \cap C$. Then, we bound  $d(v_i, u)^p$ in terms of $R_i$ and $R_j$. Note that in this case $S_i \cap S_j \neq \emptyset$, however, their intersection might be a facility $u'$ different from $u$. So, by the approximate triangle inequality (see Corollary~\ref{cor:gen-approx-triangle})
        \begin{align*}
            d(v_i, u)^p 
            \le 3^{p-1} (d(v_i, u')^p + d(u', v_j)^p + d(v_j, u)^p) \le 3^{p-1} (R_i + 2 R_j) \le 3^p \cdot R_i
        \end{align*}
    \end{itemize}
    Hence, the total cost of such clustering is at most 
    \begin{align*}
        \sum_{v_i\in D'} w'(v) \cdot d(v_i, C)^p \le 3^p \cdot \sum_{v_i\in D'} R_i \le \big(\frac{9}{2}\big)^p \cdot \cost_{\stlp}(y), 
    \end{align*}
    where the last inequality follows from Lemma~\ref{lem:str-to-fair-lp}.
\end{proof}
\noindent\textbf{Step 2. Constructing an integral solution.} 
Finally, we show that we can always find an integral solution of $\stllp$ that picks at least one center from every set $S_1, \cdots, S_L$. Note that by showing the existence of such a solution, automatically (via Lemma~\ref{lem:approximation-bound}) we have the guarantee that it is a $3^p$-approximation solution of $\stllp$. We show the existence of such an integral solution via an application of max-flow problem.

\begin{lemma}\label{lem:max-flow}
    Given a collection of disjoint sets $S_1, \cdots, S_L$ where $L \le k$, there exists an integral solution that picks a set of $k$ centers $C$ with the following extra properties:
    \begin{itemize}
        \item For every $j\in [L]$, $C\cap S_j \neq \emptyset$, and
        \item For every group $i\in [\ell]$, $\alpha_i \le |C \cap P_i| \le \beta_i$.
    \end{itemize}
\end{lemma}
\begin{proof}
To show the existence of such a solution, we construct the following instance of network-flow. As in Figure~\ref{fig:network_flow}, we create a network with $6$ layers. Layer 0 consists of a single source vertex $s$ and layer $1$ consists of $L+1$ vertices corresponding to sets $S_1, \cdots, S_l$ and a dummy set $\bar{S}$. Moreover, for every $i\in [L]$, the source vertex $s$ is connected to $S_i$ with an edge of capacity $1$. There is also an edge from $s$ to $\bar{S}$ with capacity $k - L$. In layer $2$, there are $|F|$ vertices corresponding to each facility in $F$. Moreover, for every $i\in [L]$, $S_i$ is connected to the facilities $u_{i_1}$ and $u_{i_2}$ where $S_i = \{u_{i_1}, u_{i_2}\}$ with capacity $1$. Moreover, $\bar{S}$ has an edge to every facility with capacity $1$. In layer 3, we have a vertex $g_i$ for every group $i\in [\ell]$. Then, the edges between layer $2$ and layer $3$ specify the membership of the facilities to the groups. Note that as groups are disjoint, each facility is connected to exactly one group. Finally, the edges between the vertices in layer $4$, representing the groups, and the vertex $t_1$ in layer $5$ assures that the number of facilities selected from each group $i$ is in the given integral $[\alpha_i, \beta_i]$: there is an edge with respectively lower and upper  capacities of $\alpha_i, \beta_i$ connecting $g_i$ to $t_1$. Finally, there is an edge from $t_1$ to $t_2$ with capacity $k$ to bound the total number of selected centers.

It is straightforward to check that if there exists an integral flow of value $k$ from $s$ to $t_2$, then the set of facilities whose corresponding vertices in layer $2$ receive a unit of flow is a set of centers that satisfies all requirements of fair range clustering. So, it remains to show that this network has a flow of value $k$. Since all capacities are integral, by the integrality theorem of max-flow, this implies that an integral flow of value $k$ exists too. Hence, it suffices to show that the network has a (possibly fractional) flow of value $k$. Consider the half-integral optimal solution of $\stllp$, $y$, from which we constructed the sets $S_1, \cdots, S_L$. By the definition of $S_1, \cdots, S_L$, we can send one unit of flow from $s$ to each $S_i$ (corresponding to $y_{u_{i_1}}$ and $\min\{y_{u_{i_2}}, 1/2\}$) then $1/2$ flow is sent from $S_i$ to each of $u_{i_1}$ and $u_{i_2}$ (which can be the same vertex too). Note that $S_i$ are all disjoint and no facility receives more than one unit of flow at the end of this step. Then, we send $k-L$ flow from $s$ to $\bar{S}$ and from $\bar{S}$ to facilities so that the vertex corresponding to each facility $u$ receives exactly $y_u$ units of flow overall. Next, each facility vertex $u$ send $y_u$ units of flow to the color class $i$ that contains $u$. Note that by feasibility of $y$ for $\stllp$, no edge capacity will be violated. Lastly, as $\sum_{u\in F} y_u =k$, the amount of flow from $t_1$ to $t_2$ constructed from solution $y$ is exactly $k$.  
\end{proof}
\begin{figure}[t]
    \centering
    \includegraphics[scale=.8]{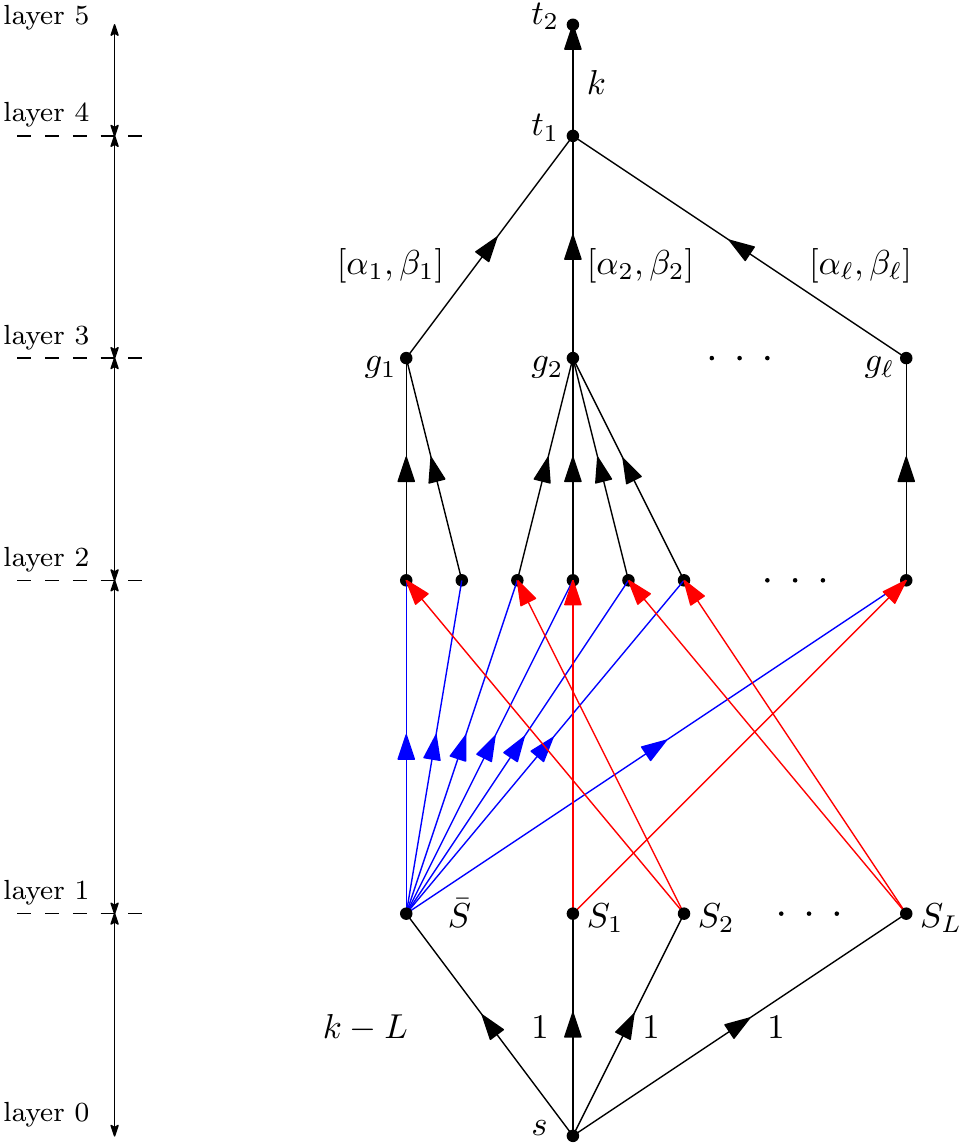}
    \caption{An example of network flow instance corresponding to a fair range clustering instance.}
    \label{fig:network_flow}
\end{figure}
Now, we are ready to prove the main theorem of this paper.
\begin{proof}[Proof of Theorem~\ref{thm:main}]
    Given an instance of the problem $\{(D, w), F\}$ with the specified set of fairness constraints, first we compute the sparse instance $(D', w')$ guaranteed in Theorem~\ref{thm:structured-fractional}. Next, we find a half-integral optimal solution $y$ of $\stllp(D', w')$ (Theorem~\ref{them:str-half-integral}). Then, by Lemma~\ref{lem:approximation-bound} and~\ref{lem:max-flow}, we can find a set of $k$ centers $C$, such that the clustering $(D', w')$ using the set of centers $C$ has cost at most $\big(\frac{9}{2}\big)^{p} \cdot \cost_{\stlp}(y)$. 
    Next, by the optimality of $y$ for $\stllp$ and by Lemma~\ref{lem:str-lp-opt-bound}, we show that the clustering cost of the instance with the set of centers $C$ is at most $\big(\frac{9}{2}\big)^{p} \cdot \cost_{\stlp}(y) = e^{O(p)} \cdot \opt_D$ where $\opt_D$ is the cost of an optimal solution of $\cllp(D,w)$.
    Finally, by Theorem~\ref{thm:sparsification-step} and raising the approximation factor to the power of $1/p$, the clustering using the center set $C$ is an $O(1)$-approximate solution on instance $(D, w)$ with the given fairness constraints $\{\alpha_i, \beta_i\}_{i\in [\ell]}$. 
\end{proof}

\section*{Conclusion} In this paper, we study the fair range clustering problem which is a generalization of several well-studied problems including fair $k$-center~\citep{kleindessner2019fair} and clustering under partition matroid. We designed efficient constant-factor approximation algorithms. Our result is the first pure multiplicative approximation algorithm for fair range clustering with general $\ell_p$-objective. 

\bibliography{fair-cls}
\bibliographystyle{abbrvnat}

\newpage
\appendix
\section{Preliminaries}
\subsection{Approximate Triangle Inequality}
\begin{lemma}[Lemma~A.1~\citep{makarychev2019performance}]\label{lem:p-norm-ineq}
Let $x, y_1, \cdots, y_n$ be non-negative real numbers and $\lambda>0, p\ge 1$. Then,
\begin{align*}
    (x + \sum_{i=1}^n y_i)^p \leq (1+\lambda)^{p-1} x^p + \Big(\frac{(1+\lambda)n}{\lambda}\Big)^{p-1} \sum_{i=1}^n y_i^p.
\end{align*}
\end{lemma}
The following approximate variants of triangle inequalities are direct corollaries of the above lemma.
\begin{corollary}\label{cor:gen-approx-triangle}
Let $(P, d)$ be a metric space. Consider distance function $d(u,v)^p$. Then, $\forall u_0,\cdots, u_r\in P, d(u_0,u_r)^p \leq r^{p-1} \cdot \sum_{i=0}^{r-1}d(u_i,u_{i+1})^p$.
\end{corollary}
\begin{proof}
Follows from the triangle inequality for the distance function $d$ and application of Lemma~\ref{lem:p-norm-ineq} with $\lambda = r-1$.
\end{proof}
\begin{corollary}\label{cor:approx-triangle}
Let $(P, d)$ be a metric space. Consider distance function $d(u,v)^p$. Then, $\forall u,v,w\in P, d(u,w)^p \leq 2^{p-1} \cdot (d(u,v)^p + d(v,w)^p)$.
\end{corollary}

\subsection{Totally Unimodular Matrices}
A matrix is called {\em totally unimodular} if every square submatrix of the given matrix has a determinant that is either $0$, $1$, or $-1$. This property has significant implications in linear programming (LP).
One of the key benefits of totally unimodular matrices is their close relationship with integer programming and LP. When a LP has a constraint matrix that is totally unimodular, it guarantees that the LP has an integral solution. Moreover, the integral solution can be found in polynomial time using various algorithms such as the ellipsoid method or the interior point method.

\begin{theorem}[\citet{ghouila1962caracterisation}]\label{thm:Ghouila-Houri}
    A matrix $A\in \mathbb{R}^{m\times n}$ is totally unimodular if and only if for every subset of the rows $R\subseteq [m]$, there is a partition of $R = R_{-} \uplus R_{+}$ such that for every $j\in [n]$,
    \[\sum_{i\in R_{+}} A_{ij} - \sum_{i\in R_{-}} A_{ij} \in \{-1, 0, 1\}.\]
\end{theorem}

\section{Constructing well-separated locations}\label{sec:preprocessing}
In this section, we follow the {\em location consolidation} approach of~\cite{charikar2002constant} and output an instance with a sparsified set of locations where pairwise distance of survived locations are ``relatively large''. To describe this step, we work with the relaxation $\cllp$.
As in our algorithm we never modify the set $F$ and the range constraints $\{\alpha_i, \beta_i\}_{i\in [\ell]}$, to specify an instance, from now on we will only specify the set of clients $(D, w)$.
Consider an optimal fractional solution $(x,y)$ of $\cllp(D, w)$---{\em throughout the paper we refer to the cost of an optimal solution of this LP as $\opt_D$}. 
For any location $v\in D$, we define $\sR(v):= \big(\sum_{u\in F} x_{vu} \cdot d(v,u)^p\big)^{1/p}$ as the {\em fractional distance} of a unit of demand at location $v$ w.r.t. the optimal solution $(x,y)$. 
Note that if $(x,y)$ is an integral solution, then $\sR(v)$ is simply the distance from $v$ to its closest open facility, which is specified by $y$.
Next, we process the locations as follows. We sort locations in a non-decreasing order of their fractional distance; we index the locations in $D$ as $v_1,\cdots, v_n$ such that $\sR(v_1) \le \sR(v_2) \le \cdots \le \sR(v_n)$. 
We iterate over the locations in this order and for every $i\in [n]$, when we are processing $v_i$, we check for all locations $v_j$ s.t. $j>i$ and $d(v_i, v_j) \le 2^{1+1/p}\cdot \sR(v_j)$. For each such location, we move the demand at location $v_j$ to $v_i$ and set the demand at location $v_j$ to zero. 
At the end of this step, the algorithm returns a new demand function $w'$ supported on $D' \subseteq D$.

\begin{algorithm}[!ht]
	\begin{algorithmic}[1]
		\STATE {\bfseries Input:} $(x, y)$ is an optimal solution of $\cllp(D, w)$
		\STATE $\sR(v) \leftarrow (\sum_{u\in F} d(v,u)^p \cdot x_{vu})^{1/p}$ \textbf{for all} $v\in D$	
        \STATE $w'(v) \leftarrow w(v)$ \textbf{for all} $v\in D$
		\STATE {\bf sort} locations in $D$ so that $\R(v_1) \leq \cdots \leq \R(v_n)$
		\FOR{$i = 1$ to $n-1$}
		  \IF{$w'(v_i) > 0$}	
                \FOR{$j = i+1$ to $n$}
				    \IF{$w'(v_j)>0$ \textbf{~and~} $d(v_i, v_j) \leq 2^{1+1/p}\cdot \sR(v_j)$}
                            \STATE $w'(v_i) \leftarrow w'(v_i) + w(v_j)$ and $w'(v_j) \leftarrow 0$
				    \ENDIF
			\ENDFOR
		  \ENDIF
        \ENDFOR	
    \end{algorithmic}
	\caption{Consolidating Locations.}
	\label{alg:client-consolidation}
\end{algorithm}
For every location $v\in D'$, we define the {\em ball} $\sB(v):= \{u\in F | d(v, u) \le 2^{1/p}\sR(v)\}$ to denote the set of facilities at distance at most $2^{1/p}\sR(v)$ from $v$. The above claim implies the following lemma.
\begin{lemma}\label{lem:half-contribution}
For every $v\in D'$, $\sum_{u\in \sB(v)} x_{vu} \ge 1/2$. 
\end{lemma}
\begin{proof}
Note that 
\begin{align*}
    \sR(v)^p 
    &= \sum_{u\in \sB(v)} x_{vu} d(u,v)^p + \sum_{u\in F\setminus \sB(v)} x_{vu} d(u,v)^p \ge \sum_{u\in \sB(v)} x_{vu} d(u,v)^p + (1- \sum_{u\in \sB(v)} x_{vu}) \cdot 2 \sR(v)^p
\end{align*}
which proves that $\sum_{u\in \sB(v)} x_{vu} \ge 1/2$. 
\end{proof}
Next, we follow the reduction steps in~\citep{krishnaswamy2011matroid,swamy2016improved} to construct a fractional solution of $\cllp$ with further structures. 
We construct a fractional solution $(x',y)$ in which for each location $v$, its demands are served by the facilities in a {\em super ball} $\sP(v)$ with fraction $\gamma \ge 1/2$ and by the facilities in $\sP(v')$ with fraction $1-\gamma_u$ where $v' := \nn_{D'}(v)$ denotes the nearest location in $D'$ (other than $v$ itself) to $v$ and the collection of super balls $\{\sP(v)\}_{v\in D'}$ are disjoint. 

We start with the feasible solution $(x^1, y)$ of $\cllp(D', w')$ where $x^1_{vu} = x_{vu}$ for all $v\in D'$ and $u\in F$ and $(x,y)$ is the optimal solution of $\cllp(D, w)$ from which the sparsified set of locations $(D', w')$ is constructed. Note that by Theorem~\ref{thm:sparsification-step}-\ref{prop:feasibility}, $(x^1, y)$ is a feasible solution of $\cllp(D', w')$ with cost at most $\opt_D$. 

\paragraph{Private facilities.} In Theorem~\ref{thm:sparsification-step}-\ref{prop:separatedness}, we showed the collection of balls $\{\sB(v)\}_{v\in D'}$ are disjoint. Here, we construct a solution $(x^2, y)$ of $\cllp(D', w')$ with cost $e^{O(p)}\cdot \opt_D$ such that each facility $u\in F\setminus \bigcup_{v\in D'} \sB(v)$, serves the clients of at most one location. In other words, each (fractionally) opened facility outside the balls {\em only} serves one location. For each $v\in D'$, the super ball $\sP(v)$ consists of $\sB(v)$ and the private facilities of $v$. 

Consider a facility $u$ that does not belong to any ball in $\{\sB(v)\}_{v\in D'}$ and serves clients from more than one locations $v_1, \cdots, v_r$. Lets assume locations are ordered according their distance to $u$; $d(v_1, u) \le \cdots \le d(v_r, u)$.

\begin{claim}\label{clm:private-facility}
For every $j\in \{2,\cdots, r\}$ and every facility $u'\in \sB(v_1)$, $d(v_j, u') \leq 3d(v_j, u)$.
\end{claim}
\begin{proof}
Since for every $j\in \{2,\cdots, r\}$, $d(v_j, u) \ge d(v_1, u)$,
\begin{align}\label{eq:private-facility}
    d(v_j, v_1) \le d(v_j, u) + d(v_1, u) \le 2 d(v_j, u).
\end{align}
Moreover, for a facility $u'\in \sB(v_1)$,
\begin{align*}
    d(v_j, u') 
    &\le d(v_j, v_1) + d(v_1, u') \\
    &\le \frac{3}{2}\cdot d(v_j, v_1)  &&\rhd\text{Theorem~\ref{thm:sparsification-step}-\ref{prop:separatedness}} \\
    &\le 3d(v_j, u) &&\rhd\text{Eq.~\eqref{eq:private-facility}}
\end{align*}
\end{proof}

Now, in $(x^2, y)$, for every $j\in \{2,\cdots,r\}$, we set the assignment of $v_j$ to $u$ to zero; $x^2_{v_j u} =0$ and instead increase the assignment of $v_j$ to facilities $u\in \sB(v_1)$ with total increase of $x^1_{v_j u}$. A formal procedure of this reassignment step is described in Algorithm~\ref{alg:private-facility-reassignment}.  
\begin{algorithm}[!ht]
	\begin{algorithmic}[1]
		\STATE {\bfseries Input:} $(x^1, y)$ is a feasible solution of $\cllp(D',w')$
        
        \STATE $x^2_{vu} \leftarrow x^1_{vu}$ \textbf{for all} $v\in D', u\in F$
        
        \FORALL{$u\in F \setminus \bigcup_{v\in D'} \sB(v)$ with $y(u)>0$}
            \STATE {\bf sort} locations $\{v\in D' \sep x^1_{vu}>0\}$ according to their distance to $u$ so that $d(v_1,u) \le \cdots \le d(v_r, u)$
            \FOR{$j=2$ to $r$}
                \STATE $x^2_{v_j u} \leftarrow 0$ and $b = x^1_{v_j u}$
                \FORALL{$u'\in \sB(v_1)$}
                    \STATE $x^2_{v_j u'} \leftarrow \min(y_{u'}, b + x^1_{v_j u'})$ and $b \leftarrow b - \min(y_{u'} - x^1_{v_j u'}, b)$
                \ENDFOR
            \ENDFOR
        \ENDFOR
    \end{algorithmic}
	\caption{Assignments to Private Facilities.}
	\label{alg:private-facility-reassignment}
\end{algorithm}
\begin{lemma}\label{ref:reassignment}
$(x^2,y)$ is a feasible solution of $\cllp(D', w')$ with cost at most $3^{p} \cdot \opt_D$.
\end{lemma}
\begin{proof}
First we 
prove the feasibility of $(x^2, y)$. 
\begin{align*}
    \sum_{u_1\in \sB(v_1)} y_{u_1} 
    &\ge \sum_{u_1\in \sB(v_1)} x^1_{v_1 u_1} &&\rhd\text{by feasibility of $(x^1, y)$} \\
    &\ge 1/2 &&\rhd\text{Lemma~\ref{lem:half-contribution}}
\end{align*}
By another application of Lemma~\ref{lem:half-contribution}, for every $j\in \{2,\cdots, r\}$,
\begin{align*}
    x^1_{v_j u} + \sum_{u'_j\in F\setminus \sB(v_j)} x^1_{v_j u'_j}
    \le 1 - \sum_{u'_j\in \sB(v_j)} x_{v_j u'_j} \le 1/2 
\end{align*}
Hence, the facilities in $\sB(v_1)$ have enough slack to accommodate extra $x^1_{v_j u}$ in their assignment from $v_j$. This implies that $(x^2, y)$ as constructed in Algorithm~\ref{alg:private-facility-reassignment} is a feasible solution of $\cllp(D', w')$.

Moreover, since by Claim~\ref{clm:private-facility} for every $j\in \{2, \cdots, r\}$ and $u'\in \sB(v_1)$, $d(v_j,u') \le 3 d(v_j, u)$, $\cost(x^2, y) \le 3^p \cdot \cost(x^1, y) = 3^p\cdot \opt_D$.
\end{proof}

\section{Omitted Proofs of Section~\ref{sec:decription}}\label{sec:proof-sec-2}
\begin{proof}[Proof of Theorem~\ref{thm:sparsification-step}]
\ref{prop:separatedness}: Wlog, lets assume that $v_1$ located before $v_2$ in the ordering considered by Algorithm~\ref{alg:client-consolidation}; hence, $\sR(v_2) \ge \sR(v_1)$. However, if $d(v_1,v_2) \ge 2^{1+1/p}\sR(v_2)$, then Algorithm~\ref{alg:client-consolidation} moves the demand of $v_2$ to $v_1$ and $v_2$ will not survive the process which is a contradiction. Thus, $d(v_1,v_2)\le 2^{1+1/p}\sR(v_2) = 2^{1+1/p}\cdot \max\{\sR(v_2),\sR(v_1)\}$.

\ref{prop:feasibility}: In~$\cllp$, no constraint depends on demands and thus the constraints of $\cllp(D', w')$ are subsets of the ones in $\cllp(D, w)$. Thus, $(x,y)$ is a feasible solution of $\cllp(D', w')$ too.
    
Next, we show that the cost of $(x,y)$ w.r.t.~$\cllp(D',w')$ is not more that the cost $(x,y)$ w.r.t.~$\cllp(D, w)$. For each location $v$, let $\bar{v}$ denote the location in $D'$ that receives the demand of $v$ by the end of Algorithm~\ref{alg:client-consolidation}.
\begin{align*}
    \cost(x,y; D',w') := \sum_{q\in D'} w'(q) \sum_{u\in F} x_{qu} d(q,u)^p 
    &= \sum_{q\in D'} w'(q) \sR(q)^p \\
    &= \sum_{q\in D'} \sum_{v\in D: \bar{v} = q} w(v) \sR(q)^p \\
    &= \sum_{v\in D} w(v) \sR(\bar{v})^p \\
    &\le \sum_{v\in D} w(v) \sR(v)^p = \opt_D
\end{align*}

\ref{prop:cost-preservation}: Algorithm~\ref{alg:client-consolidation} may either move the demand of a location $v$ to some other location $\bar{v}$ or keep it at $v$. Let $v' = \bar{v}$ in the former case and $v' = v$ in the latter case. In either case $d(v, v') \le 2^{1+1/p}\cdot \sR(v)$. Therefore, by the approximate triangle inequality (Corollary~\ref{cor:approx-triangle}),
\begin{align*}
    d(v, C)^p \le 2^{p-1} (d(v, v')^p + d(v', C)^p) \le 2^{2p} \sR(v)^p + 2^{p-1}d(v', C)^p
\end{align*}
Summing over all locations,
\begin{align*}
    \sum_{v\in D} w(v) d(v, C)^p \le \sum_{v\in D} w(v) \big(4^p \sR(v)^p +  2^{p-1} d(v', C)^p\big) \le 4^p \cdot \opt_D + 2^{p-1} \cdot z   
\end{align*}
\end{proof}
\begin{proof}[Proof of Theorem~\ref{thm:structured-fractional}]
We initialize $(\bar{x}, \bar{y})$ to the solution $(x^2,y)$ as constructed in Algorithm~\ref{alg:private-facility-reassignment}; $\bar{x} = x^2$ and $\bar{y} = y$. We keep $\bar{y} = y$ throughout the process but modify $\bar{x}$ to satisfy the desired properties.

By Claim~\ref{clm:private-facility}, we can partition open facilities, $\{u\in F \sep y_{u} > 0\}$ into disjoint super balls $\{\sP(v)\}_{v\in D'}$. So,~\ref{prop:ball-inclusion} will be satisfied by $\bar{y}$.

Next, we modify the assignment vector $x_2$ so that for every $v\in D'$ and $u\in \sP(v) \setminus \sB(v)$, $\bar{x}_{vu} >0$ only if $\sum_{u\in \sB(v)} \bar{y} < 1$. Similarly, for every $v\in D'$ and $u\in F\setminus \sP(v)$, $\bar{x}_{vu}>0$ iff $\sum_{u\in \sP(v)} \bar{y}_{v} < 1$. So, by the construction of $(\bar{x},\bar{y})$,~\ref{prop:optimality} is satisfied.

Consider a location $v\in D'$. Since the total $\bar{y}$ contributions in each of $\sB(v)$ and $\sB(v')$ is at least half, in $\bar{x}$ we assign the remaining demand of each location $v$, which are not satisfied by the facilities in $\sP(v)$, to the facilities in $\sB(v')$. So, $(\bar{x},\bar{y})$ satisfies~\ref{prop:external-center} and~\ref{prop:half-center}.

Without loss of generality, we can assume that~\ref{prop:near-center-further-than-private-center} holds, otherwise, we could transfer the $\bar{x}_{vu}$ fraction of assignment of $v$ from $u$ to a facility in $\sB(v')$, set $\bar{y}_u = 0$, and reduce the total cost. Again, this reassignment is always possible because for every $v\in D'$,
\[\bar{x}_{vu} + \sum_{u'\in F\setminus \sP(v)} \bar{x}_{vu'} \le 1/2 \le \sum_{u'\in \sB(v')} \bar{y}_{u'}.\]

Finally, the last property follows from Claim~\ref{clm:private-facility} and the following bound. For every $u''\in \sB(v'')$ and $u'\in \sB(v')$ where $w\notin \{v, v'\}$,
\begin{align*}
    d(v,u') 
    &\le d(v, v') + d(v', u') &&\rhd\text{triangle inequality} \\
    &\le (1 + \frac{1}{2}) \cdot d(v,v') \\ 
    &\le (1 + \frac{1}{2}) \cdot d(v,v'') &&\rhd d(v,v') \le d(v,v'')\\
    & \le 3 \cdot d(v,u''),
\end{align*}
where the second inequality holds since $d(v,v') \ge 2^{1+\frac{1}{p}}\sR(v')$ and $d(v',u')\le 2^{\frac{1}{p}}\sR(v')$. Similarly, the fourth inequality follows from $d(v,v'') \ge 2^{1+\frac{1}{p}}\sR(v'')$ and $d(v'',u'')\le 2^{\frac{1}{p}}\sR(v'')$ (which implies that $d(v,u'')\ge \frac{1}{2}d(v,v'')$). Thus, $\cost(\bar{x}, \bar{y}) \le 3^{p} \cdot \cost(x^2, y) \le 9^p \cdot \opt_D$.  
\end{proof}
\end{document}